\documentclass{amsproc}

\usepackage{graphicx}
\usepackage{subfig}
\usepackage{amsmath}
\usepackage{amsfonts}
\usepackage{amssymb}%
\usepackage{amsthm}
\usepackage{amscd}
\usepackage{url}
\usepackage[all]{xypic}
\usepackage{xy} 

\newcommand{\mcA}{\mathcal{A}}

\newcommand{\PP}{\mathbb{P}}

\newcommand{\C}{\mathbb{C}}
\newcommand{\R}{\mathbb{R}}

\newcommand{\rank}{\operatorname{rank}}

\newcommand{\ceil}[1]{\ensuremath{\lceil #1 \rceil}}
\newcommand{\floor}[1]{\ensuremath{\lfloor #1 \rfloor}}
\newcommand {\TP} {\mathbb{TP}}
\newcommand{\un}{\underline{n}}

\usepackage{hyperref}

\theoremstyle{plain}
\newtheorem{theorem}{Theorem}[section]
\newtheorem{corollary}[theorem]{Corollary}
\newtheorem{proposition}[theorem]{Proposition}
\newtheorem{lemma}[theorem]{Lemma}
\newtheorem{remark}[theorem]{Remark}
\newtheorem{conjecture}[theorem]{Conjecture}
\newtheorem{question}[theorem]{Question}
\theoremstyle{definition}
\theoremstyle{definition}

\theoremstyle{definition}
\newtheorem{example}[theorem]{Example}

\makeatletter
\@namedef{subjclassname@2010}{%
 \textup{2010} Mathematics Subject Classification}
\makeatother

\title{Geometry of the Restricted Boltzmann Machine}

\author{Mar\'ia Ang\'elica Cueto}
\address{Department of Mathematics, University of California,
Berkeley, CA 94720, USA}
\email{macueto@math.berkeley.edu}

\author{Jason Morton}
\address{Department of Mathematics, Stanford University,
Stanford, CA 94305, USA}
\email{jason@math.stanford.edu}

\author{Bernd Sturmfels} \address{Department of Mathematics,
  University of California, Berkeley, CA 94720, USA}
\email{bernd@math.berkeley.edu} \thanks{Mar\'ia Ang\'elica Cueto was
  supported by a UC Berkeley Chancellor's Fellowship.  Jason Morton
  was supported in part by DARPA grant HR0011-05-1-0007 and NSF grant
  DMS-0354543.  Bernd Sturmfels was supported in part by NSF grants
  DMS-0456960 and DMS-0757236.}

\begin{document}

\begin{abstract}
The restricted Boltzmann machine is a graphical
model for binary random variables.
Based on a complete bipartite graph separating
hidden and observed variables, it is the binary
analog to the factor analysis model.
We study this graphical model from the perspectives of
algebraic statistics and tropical geometry,
starting with the observation that its Zariski
closure is a Hadamard power of the first secant variety
of the Segre variety of projective lines. We derive
a dimension formula for the tropicalized model,
and we use it to show that the restricted Boltzmann machine 
is identifiable in many cases.
Our methods include coding theory and geometry of linear threshold functions.
\end{abstract}

\maketitle

\section{Introduction}
\label{sec:introduction}

A primary focus in  algebraic statistics is the study of
statistical models that can be represented by
polynomials in the model parameters.
This class of algebraic statistical models includes graphical models for
both Gaussian and discrete random variables \cite{DrtonSullivant, Oberwolfach}.
In this article we study a  family of
binary graphical models  with hidden variables.
The underlying graph is the complete bipartite graph $K_{k,n}$:

\begin{figure}[htb]
\begin{center}
\[
\begin{xy}<15mm,0mm>:
(0,0) *{\circ};
(0,-1) *{\bullet} **@{-},
(.5,-1) *{\bullet} **@{-},
(1,-1) *{\bullet} **@{-},
(1.5,-1) *{\bullet} **@{-},
(2,-1) *{\bullet} **@{-},
(1,0) *{\circ};
(0,-1) *{\bullet} **@{-},
(.5,-1) *{\bullet} **@{-},
(1,-1) *{\bullet} **@{-},
(1.5,-1) *{\bullet} **@{-},
(2,-1) *{\bullet} **@{-},
(2,0) *{\circ};
(0,-1) *{\bullet} **@{-},
(.5,-1) *{\bullet} **@{-},
(1,-1) *{\bullet} **@{-},
(1.5,-1) *{\bullet} **@{-},
(2,-1) *{\bullet} **@{-},
(-1.5,-.35) *{\text{state}};
(-1.5,-.65) *{\text{vectors}};
(-.5,0) *{h};
(-.5,-1) *{v};
(0,.3).(2,.3)!C *\frm{^\}},+U*++!D\txt{$k$ hidden variables};
(0,-1.3).(2,-1.3)!C *\frm{_\}},+U*++!U\txt{$n$ observed variables};
(4,-.5) *{\text{parameters}};
(2.5,0) *{c};
(2.5,-1) *{b};
(2.5,-.5) *{W};
\end{xy}
\]
\end{center}
\caption{Graphical representation of the restricted Boltzmann machine.} \label{fig:RBM}
\end{figure}
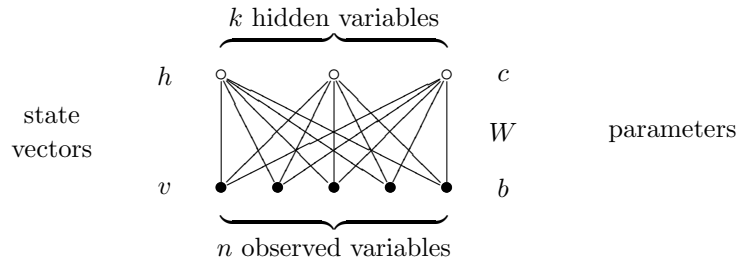

The $k$ white nodes in the top row of Figure \ref{fig:RBM}  
represent hidden random variables. The
 $n$ black nodes in the bottom row
represent observed random variables.
The {\em restricted Boltzmann machine}  (RBM)
is the undirected graphical model for binary random variables
specified by this bipartite graph. We identify the model with its set of joint distributions 
which is a subset $M_n^k$ of the probability simplex $\Delta_{2^n-1}$.

The graphical model for Gaussian random variables represented
by Figure \ref{fig:RBM} is the
{\em factor analysis} model, whose algebraic properties were studied in
\cite{BD, FactorAnalysis}. Thus, the restricted Boltzmann machine is the 
binary undirected analog of factor analysis. Our aim here is to study this 
model from the perspectives of algebra and geometry. Unlike in the factor
analysis study \cite{FactorAnalysis}, an important role will now be played by
{\em tropical geometry}  \cite{TGSM}. This was already seen for  $n=4$ and $k=2$ in
the solution by Cueto and Yu \cite{Cueto}
of the implicitization challenge  in 
\cite[Problem 7.7]{Oberwolfach}.

The restricted Boltzmann machine has been the subject of a recent
resurgence of interest due to its role as the building block of the
deep belief network.  Deep belief networks are designed to learn
feature hierarchies to automatically find high-level representations
for high-dimensional data.  A deep belief network comprises a stack of
restricted Boltzmann machines.  Given a piece of data (state of the
lowest visible variables), each layer's most likely hidden 
states are treated as data for the next layer. A new effective training
methodology for deep belief networks, which begins by training each
layer in turn as an RBM using contrastive divergence, was introduced
by Hinton et al. \cite{Hinton2006}. This method led to many new
applications in general machine learning problems including object
recognition and dimensionality reduction \cite{HintonScience}.
While promising for practical applications, 
the scope and basic properties of these statistical models have
only begun to be studied.  For example, Le Roux and Bengio
\cite{LeRoux2008} showed that any distribution with support on $r$ visible states may be 
arbitrarily well approximated provided there are at least $r+1$
hidden nodes. Therefore, any distribution can be approximated with $2^n+1$ hidden nodes. 

The question which started this project is
whether the restricted Boltzmann machine is identifiable.
The dimension of the fully observed binary graphical model on
 $K_{k,n}$ is equal to  $nk+n+k$,  the number of nodes plus
 the number of edges.
 We conjecture that this dimension is preserved under the
 projection corresponding to the algebraic elimination of the $k$ hidden variables.
 Here is the precise statement:

\begin{conjecture}
\label{conj:main} 
 The restricted Boltzmann machine has the expected dimension, i.e.~$M_n^k$ is a semialgebraic set of dimension 
  ${\rm min}\{nk+n+k,2^n-1\}$ in $\Delta_{2^n-1}$.
   \end{conjecture}
 
This conjecture is shown to be true in many special cases. In particular, it
holds for all $k$ when $n+1$ is a power of $2$. This is a consequence
of the following:

\begin{theorem}
\label{thm:looksreasonable}
 The restricted Boltzmann machine has the expected dimension
 ${\rm min}\{nk+n+k,2^n-1\}$  when $k \leq 2^{n-\ceil{\log_2 (n+1)}}$ and when
 $k \geq 2^{n-\floor{\log_2 (n+1)}}$.
\end{theorem}

We note that Theorem \ref{thm:looksreasonable} covers most cases of restricted Boltzmann machines as used in practice, as those generally satisfy $\,k \leq 2^{n-\ceil{\log_2 (n+1)}}$.
The case of  large $k$ is primarily of theoretical interest
and has been studied recently in \cite{LeRoux2008}. 

This paper is organized as follows.  In Section~2 we introduce four
geometric objects, namely, the RBM model, the RBM variety, the
tropical RBM model, and the tropical RBM variety, and we formulate a
strengthening of Conjecture \ref{conj:main}.  Section~3 is concerned
with the case $k=1$.  Here the RBM variety is the variety of secant
lines of the Segre variety $(\PP^1)^n \subset \PP^{2^n-1}$.  The
general case $k > 1$ arises from that secant variety by way of a
construction we call the {\em Hadamard product of projective
  varieties}, as shown in Proposition \ref{prop:Hadamard}.  In Section~4 we analyze the tropical RBM model, we establish a formula for its
dimension (Theorem \ref{thm:cornercuts}), and we draw on results from
coding theory to derive Theorem \ref{thm:looksreasonable} and Table
\ref{tab:knownspecialcases}.  In Section~5 we study the
piecewise-linear map that parameterizes the tropical RBM model.  The
{\em inference functions} of the model (in the sense of
\cite{InferenceFunctions, TGSM}) are $k$-tuples of {\em linear
  threshold functions}.  We discuss the number of these functions.
  Figure \ref{fig:TM^1_3}
  shows the combinatorial structure of the
tropical RBM model for $n{=}3$ and $k{=}1$.

\section{Algebraic Varieties, Hadamard Product and Tropicalization}
\label{sec:algebr-vari-hadam}

We begin with an alternative definition of the restricted Boltzmann machine.
This ``machine'' is a statistical model for binary random variables where $n$ of the variables
are visible and $k$ of the variables are hidden. The states of the hidden 
and visible variables  are written as binary 
vectors $h \in \{0,1\}^k$ and $v \in \{0,1\}^n$ respectively.
We introduce $nk+n+k$ model parameters, namely, 
the entries of a real $k \times n$ matrix $W$ and
the entries of two vectors $b \in \R^n$ and $c \in \R^k$, and we set
\begin{equation}
\label{eq:psi}
 \psi(v,h) \,\,\, = \,\,\, \exp(h^{\top}Wv + b^{\top}v + c^{\top}h) .
 \end{equation}
The probability distribution on the visible random variables in our model equals
\begin{equation}
\label{eq:distribution}
p(v)\,\,\, =  \,\,\,\, \frac{1}{Z} \,\cdot \!\! \sum_{h \in \{0,1\}^k} \psi(v,h),
\end{equation}
 where $Z= \sum_{v,h} \psi(v,h)$ is the {\em partition function}.
  We denote by $M_n^k$ the subset of the open probability simplex
   $ \Delta_{2^n-1}$ consisting of all such distributions 
   $(p(v):v \in \{0,1\}^n)$ as the parameters $W,b$ and $c$ run over $\R^{k \times n}$,
   $\R^{n}$ and $\R^k$ respectively.
   
   In what follows we refer to $M_n^k$ as the {\em RBM model} with $n$ visible nodes and $k$ hidden nodes. It coincides with the binary graphical model associated with the complete bipartite graph 
   $K_{k,n}$ as  described in the Introduction. This is indicated
    in Figure \ref{fig:RBM} by   the labeling with the states
 $v, h$ and the model parameters $c,W,b$.
  
 The parameterization in (\ref{eq:psi}) is not polynomial because it involves
 the exponential function. However, it is equivalent to the polynomial parameterization
 obtained by replacing each model parameter by its value under the exponential function:
 $$  \gamma_i = {\rm  exp}(c_i) \,,\,\,\,
 \omega_{ij} = {\rm exp}(W_{ij}) \,,\,\,\,
 \beta_j = {\rm exp}(b_j) .$$
 This coordinate change translates (\ref{eq:psi}) into the squarefree monomial
 $$ \psi(v,h) \,\,\, = \,\,\,
\prod_{i=1}^k \gamma_i^{h_i}
\cdot 
  \prod_{i=1}^k \prod_{j=1}^n \omega_{ij}^{h_i v_j}
  \cdot
  \prod_{j=1}^n \beta_j^{v_j},
  $$
  and we see that the probabilities in (\ref{eq:distribution}) can be factored as follows:
 \begin{equation}    \label{eq:distribution2}
    p(v) \,\,\ = \,\,\,\,
  \frac{1}{Z} \, 
   \beta_1^{v_1}\beta_2^{v_2} \cdots \beta_n^{v_n}
   \prod_{i=1}^k \bigl(1  + \gamma_i \,
   \omega_{i1}^{v_1}
    \omega_{i2}^{v_2}\cdots
    \omega_{in}^{v_n} \bigr) 
    \qquad \hbox{for} \quad v \in \{0,1\}^n.
    \end{equation}
The RBM model $M^k_n$ is the image of the  polynomial map
   $\,   \R_{>0}^{nk+k+n} \rightarrow \Delta_{2^n-1}\,$
   whose $v$th coordinate equals (\ref{eq:distribution}).
   This shows that $M^k_n$ is a semialgebraic subset of $\Delta_{2^n-1}$.

\smallskip

When faced with a high-dimensional semialgebraic set arising in statistics, it is often
useful to simplify the situation by disregarding all inequalities and by
replacing the real numbers $\R$ by the complex numbers $\C$.
This leads us to considering the Zariski closure  $V^k_n$ of the RBM model $M^k_n$. 
This is the algebraic variety in the complex projective space $\mathbb{P}^{2^n-1}$
parameterized by (\ref{eq:distribution2}).
We call $V^k_n$ the {\em RBM variety}.

\medskip

Given any two subvarieties $X$ and $Y$ of a projective space $\mathbb{P}^m$, we define
their {\em Hadamard product} $X * Y$ to be the closure of the image of the rational map
$$ X \times Y \dashrightarrow \mathbb{P}^m \,,\,\,
 (x,y) \mapsto  (x_0 y_0 : x_1 y_1 : \ldots : x_m y_m ) .$$ 
 For any projective variety $X$, we may consider its Hadamard square
$X^{[2]} = X * X$ and its higher Hadamard powers
$ X^{[k] } \, = \,X * X^{[k-1]}$. 
If $M$ is a subset of the open simplex $\Delta_{m-1}$
then its Hadamard powers $M^{[k]}$ are also defined
by componentwise multiplication followed by rescaling so that the
coordinates sum to one. This construction is compatible with
taking Zariski closures, i.e.~we have
$\overline{M^{[k]}} = 
\overline{M}^{[k]}$.

In the next section we shall take a closer look at the case $k=1$, and we shall recognize
$V^1_n $ as a secant variety and $M^1_n$ as a phylogenetic model. Here, we prove
that the case of $k > 1$ hidden nodes reduces to $k=1$ using
Hadamard powers.

\begin{proposition} \label{prop:Hadamard}
The RBM variety and model factor as Hadamard powers:
$$ V^k_n \,=\, (V^1_n)^{[k]} \quad \hbox{and} \quad M^k_n \,=\, (M^1_n)^{[k]}.$$
\end{proposition}

\begin{proof}
A strictly positive vector $p$ with coordinates $p(v)$ as in (\ref{eq:distribution2}) 
admits a componentwise factorization into similar vectors for $k=1$, and, conversely,
the componentwise product of $k$ probability distributions in $M^1_n$ becomes
a distribution in $M^k_n$ after division by the partition function.
Hence $\, M^k_n = (M^1_n)^{[k]}$ in $\Delta_{2^n-1}$. The equation  
$ V^k_n = (V^1_n)^{[k]} $ follows by passing to the Zariski closure
in $\mathbb{P}^{2^n-1}$.
\end{proof}

The emerging field of {\em tropical mathematics} is predicated on the idea that
$\, {\rm log}({\rm exp}(x) + {\rm exp}(y)) \,$ is approximately equal to
${\rm max}(x,y)$ when $x$ and $y$ are quantities of different scale. 
For a first introduction see \cite{TropMath},
and for further reading see \cite{CTV, Develin, tropicalDraisma, MY} and references therein.
The process of passing from ordinary arithmetic to the max-plus algebra
is known as {\em tropicalization}.
 The same approximation motivates the definition of the {\it softmax} function in the neural networks literature.   A statistical perspective is offered in work by 
 Pachter and the third author \cite{ASCB, TGSM}.
 
If $q(v)$ approximates $ {\rm log}(p(v))$ in the 
 sense of tropical mathematics,
 and if we disregard the global additive constant $- \log Z$,
 then  (\ref{eq:distribution}) 
translates into the formula
\begin{equation}
\label{qformula}
 q(v) \,\,\, = \,\,\, {\rm max} \bigl\{\,
h^{\top}Wv + b^{\top}v + c^{\top}h \, :\, h \in \{0,1\}^k \,\bigr\} .
\end{equation}
This expression is a piecewise-linear concave function
$ \R^{nk+n+k} \rightarrow \R$ on the space of  model parameters $(W,b,c)$.
As $v$ ranges over $\{0,1\}^n$, there are $2^n$ such concave functions,
and these form the coordinates of  a piecewise-linear map 
\begin{equation}
\Phi : \R^{nk+n+k} \rightarrow \TP^{2^n-1} .\label{eq:Phi}
\end{equation}
Here $\TP^{2^n-1}$ denotes the {\em tropical projective space}
$\R^{2^n}\!/\R(1,1,\ldots,1)$, as in \cite{CTV,
  tropicalDraisma}. 
The image of the map $\Phi$ is denoted $TM^k_n$ and is called the {\em
  tropical RBM model}.  The map $\Phi$ is the \emph{tropicalization} of the
given parameterization of the RBM model. 
It is our objective to investigate its
geometric properties.

This situation fits precisely into the general scheme of parametric
maximum a posterior (MAP) inference introduced in \cite{TGSM} and
studied in more detail by Elizalde and Woods
\cite{InferenceFunctions}.  In Section~5 below, we discuss the
statistical relevance of the map $\Phi$ and we examine its geometric
properties. Of particular interest are the domains of linearity of
$\Phi$, and how these are mapped onto the cones of the model $TM^k_n$.

Finally, we define the {\em tropical RBM variety} $TV^k_n$ to be the
tropicalization of the RBM variety $V^k_n$.  As explained in \cite[\S
3.4]{ASCB} and \cite[\S 3]{TGSM}, the tropical variety $TV^k_n$ is the
intersection in $\TP^{2^n-1}$ of all the tropical hypersurfaces
$\mathcal{T}(f)$ where $f$ runs over \emph{all} polynomials that
vanish on $V^k_n$ (or on $M^k_n$).  By definition, $\mathcal{T}(f)$ is
the union of all codimension one cones in the normal fan of the Newton
polytope of $f$.  If the homogeneous prime ideal of the variety
$V^k_n$ were known then the tropical variety $TV^k_n$ could in theory
be computed using the algorithms in \cite{CTV} which are implemented
in the software {\tt Gfan} (\cite{gfan}).  However, this prime ideal
is not known in general.  In fact, even for small instances, its computation is
very hard and relies primarily on tropical geometry techniques such as
the ones developed in \cite{Cueto}.  For instance, the main result
 in \cite{Cueto} states that the RBM variety
$V^2_4 $ is a hypersurface of degree $110$ in $\PP^{15}$, and it
remains a challenge to determine a formula for the defining
irreducible polynomial of this hypersurface.
To appreciate this challenge, note that the number of
monomials in the relevant multidegree  equals
 $5\, 529\, 528\, 561\, 944$.

\smallskip

Here is a brief summary of the four geometric objects we have introduced:
 
\begin{itemize}
\item
The semialgebraic set $M^k_n \subset \Delta_{2^n-1}$
of probability distributions represented by the restricted Boltzmann machine.
We call $M^k_n$ the {\em RBM model}.
\item The Zariski closure $V^k_n$ of the RBM model $M^k_n$. This is an algebraic variety in the complex projective space $\mathbb{P}^{2^n-1}$.
We call $V^k_n$ the {\em RBM variety}.
\item The tropicalization $TV^k_n$ of the variety $V^k_n$. This is a tropical
variety in the tropical projective space $\mathbb{TP}^{2^n-1}$.
We call $TV^k_n$ the {\em tropical RBM variety}.
\item The image $TM^k_n$ of the tropicalized  parameterization $\Phi$.
This is the subset of $\mathbb{TP}^{2^n-1}$ consisting of all optimal
score value vectors in the MAP inference problem for the RBM.
We call $TM^k_n$ the {\em tropical RBM model}.
\end{itemize}

We have inclusions $M^k_n \subset V^k_n$ and $TM^k_n \subset TV^k_n$.
The latter inclusion is the content of the second statement in
\cite[Theorem 2]{TGSM}.  We shall see that both
inclusions are strict even for $k = 1$.  For example, $\,M^1_3 \,$ is
a proper subset of $\, V^1_3 \cap \Delta_{7} = \Delta_7\,$ since
points in this set must satisfy the inequality $\,\sigma_{12}
\sigma_{13} \sigma_{23} \geq 0\,$ as indicated in Theorem
\ref{thm:wishfulthinking} below.
Likewise, $TM^1_3$ is a proper subfan of $\TP^7 = TV^1_3$.
This subfan will be determined in our discussion of 
the secondary fan structure in Example \ref{ex52}.

The dimensions of our four geometric objects satisfy the
following chain of equations and inequalities:
\begin{equation}
\label{InEqChain}
 {\rm dim}(TM^k_n) \,\leq\
{\rm dim}(TV^k_n) \,=\,
{\rm dim}(V^k_n) \,=\,
{\rm dim}(M^k_n) \,\leq \,  \min\{nk + n + k, 2^n-1\}.
\end{equation}
Here, the tropical objects $TM^k_n$ and $TV^k_n$ are polyhedral fans,
and by their dimension we mean the dimension of any cone of maximal
dimension in the fan. When speaking of the dimension of $V^k_n$ we
mean the Krull dimension of the projective variety, and for the model
$M^k_n$ we mean its dimension as a semialgebraic set.

The leftmost inequality in (\ref{InEqChain}) holds because $TM^k_n
\subset TV^k_n$. The left equality holds by the
Bieri-Groves Theorem (cf.~\cite[Theorem 4.5]{tropicalDraisma}) which
ensures that every irreducible variety has the same dimension as its
tropicalization.  The second equality follows from standard real
algebraic geometry results because $M_n^k$ has a regular point and is
Zariski dense in $V_n^k$.  Finally, the rightmost inequality in
(\ref{InEqChain}) is seen by counting parameters in the definition
(\ref{eq:psi})--(\ref{eq:distribution}) of the RBM model $M_n^k$, and
by bounding its dimension by the dimension of the ambient
space $\Delta_{2^n-1}$.

\smallskip

We conjecture that both of the inequalities in
(\ref{InEqChain}) are actually equalities:

\begin{conjecture}
\label{conj:main2} 
 The tropical RBM model has the expected dimension, i.e.~$TM_n^k$  is a 
 polyhedral fan of dimension ${\rm min}\{nk+n+k,2^n-1\}$ in $\TP^{2^n-1}$.
   \end{conjecture}
   
   In light of the inequalities (\ref{InEqChain}), Conjecture \ref{conj:main2}
   implies Conjecture \ref{conj:main}. In Section~4 we shall prove
   some special cases of these conjectures, including Theorem~\ref{thm:looksreasonable}.
   
\section{The First Secant Variety of the $n$-Cube}
\label{sec:first-secant-variety}

We saw in Proposition \ref{prop:Hadamard} that the RBM for $k \geq 2$
can be expressed as the Hadamard power of the RBM for
$k=1$. Therefore, it is crucial to understand the model with one
hidden node. In this section we fix $k=1$ and we present an analysis
of that case. In particular, we shall give a combinatorial description
of the fan $TM^1_n$ which shows that it has dimension $2n+1$, as
stated in Conjecture \ref{conj:main2}.

We begin with a reparameterization of our model that describes it as a
secant variety.  Let $\lambda$, $\delta_1,\ldots,\delta_n$,
$\epsilon_1, \ldots, \epsilon_n$ be real parameters which range over
the open interval $(0,1)$, and consider the polynomial map
$\,p:(0,1)^{2n+1} \rightarrow \Delta_{2^n-1}\,$ whose coordinates are
given by
\begin{equation}    \label{eq:distribution3}\,\,
 p(v) \,\, = \,\,\, \lambda \prod_{i=1}^n \delta_i^{1-v_i} (1-\delta_i)^{v_i}
\,+\, (1-\lambda ) \prod_{i=1}^n \epsilon_i^{1-v_i} (1-\epsilon_i)^{v_i} 
\quad \hbox{for} \,\, v \in \{0,1\}^n .  \end{equation}

\begin{proposition}\label{prop:secantDescriptionM1n}
The image of $p$ coincides with the RBM model $M^1_n$.
\end{proposition}

\begin{proof} Recall the parameterization~\eqref{eq:distribution2} of the RBM model $M^1_n$ from Section~2:
\begin{equation}    \label{eq:distribution4} \,\,
    p(v) \,\,\ = \,\,\,\,
  \frac{1}{Z} \, 
   \beta_1^{v_1}\beta_2^{v_2} \cdots \beta_n^{v_n}
   \bigl(1  + \gamma \,
   \omega_{1}^{v_1}
    \omega_{2}^{v_2}\cdots
    \omega_{n}^{v_n} \bigr) 
    \quad \, \hbox{for} \quad v \in \{0,1\}^n.
    \end{equation}
    We define a bijection between the parameter spaces
$\R_{>0}^{2n+1}$  and $(0,1)^{2n+1}$ as follows:
$$ \beta_i = \frac{1-\delta_i}{\delta_i}
\quad \hbox{and} \quad
\omega_i = \frac{\delta_i}{1-\delta_i}\frac{1-\epsilon_i}{\epsilon_i} \quad
\hbox{for} \,\,\, i = 1,2,\ldots,n , $$
$$ \gamma \,\, = \,\, Z(1-\lambda)
\epsilon_1 \epsilon_2 \cdots \epsilon_n
\quad \hbox{where} \quad
Z = (\lambda \delta_1 \delta_2 \cdots \delta_n)^{-1} . $$
This substitution is invertible and it transforms (\ref{eq:distribution4}) into  (\ref{eq:distribution3}).
\end{proof}

Proposition~\ref{prop:secantDescriptionM1n} shows that $M^1_n$ is the
first mixture of the independence model for $n$ binary random
variables. In phylogenetics, it coincides with the {\em general Markov
  model on the star tree} with $n$ leaves. A semi-algebraic
characterization of that model follows as a special case from recent
results of Zwiernik and Smith~\cite{Piotr}.  We shall present and
discuss their characterization in Theorem \ref{thm:wishfulthinking}
below.

First, however, we remark that the Zariski closure of a mixture
of an independence model is a secant variety of the corresponding
Segre variety.  This fact is well-known (see e.g. \cite[\S
4.1]{Oberwolfach}) and is here easily seen from
(\ref{eq:distribution3}). We conclude:

\begin{corollary}
The first RBM variety 
$V^1_n$ coincides with the first secant variety of the Segre embedding of 
the product of projective lines $(\PP^1)^n$ into $\PP^{2^n-1}$,
and the first tropical RBM variety $TV^1_n$ is the tropicalization of that secant variety.
\end{corollary}

We next describe the equations defining the
first secant variety $V^1_n$. The coordinate functions $p(v)$ 
are the entries of  an $n$-dimensional table of format
$2 {\times} 2 {\times} \cdots {\times} 2 $. For each set partition
$\{1,2,\ldots,n\} = A \cup B$ we can write this table as
an ordinary two-dimensional matrix of format
$2^{|A|} {\times} 2^{|B|}$, with rows indexed by
$\{0,1\}^A$ and columns indexed by $\{0,1\}^B$.
These matrices are the {\em flattenings} of the
$2 {\times} 2 {\times} \cdots {\times} 2 $-table.
Pachter and Sturmfels \cite[Conjecture 13]{TGSM} conjectured 
that the homogeneous prime ideal of the projective variety 
$V^1_n \subset \PP^{2^n-1}$ is
generated by the $3 \times 3$-minors of all the
flattenings of the table $(p(v)))_{v \in \{0,1\}^n}$.
This conjecture has been verified computationally
for $n \leq 5$. A more general form of this conjecture was 
stated in \cite[\S 7]{GSS}. The set-theoretic version of that
general conjecture was proved by  Landsberg and Manivel in
\cite[Theorem 5.1]{Landsberg}. Their results imply:

\begin{theorem}[Landsberg-Manivel]
\label{thm:flat}
The projective variety 
$V^1_n \subset \PP^{2^n-1}$ is the common zero set of
the $3 \times 3$-minors of all the
flattenings of the table $(p(v)))_{v \in \{0,1\}^n}$.
\end{theorem}

We now come to the inequalities that determine
$M^1_n$ among the real points of $V^1_n$.
For any pair of indices $i,j \in \{1,2,\ldots,n\}$
we write $\sigma_{ij}$ for the covariance
of the two random variables $X_i$ and $X_j$ obtained
by marginalizing the distribution, and we write
$\,\Sigma = (\sigma_{ij})\,$ for the $n {\times} n$-covariance matrix.
We regard $\Sigma$ as a polynomial map from 
 the simplex $\Delta_{2^n-1}$ to the space $\R^{\binom{n+1}{2}}$
 of symmetric $n {\times} n$-matrices.
The off-diagonal entries of the covariance matrix $\Sigma$ are the 
$2 {\times} 2$-minors obtained by marginalization from 
the table $(p(v))$. For example, for $n=4$ the covariances are
\begin{eqnarray*} & \sigma_{12} \, = \,
{\rm det} \begin{pmatrix}
p_{0000}{+}p_{0001}{+}p_{0010}{+}p_{0011} &
p_{0100}{+}p_{0101}{+}p_{0110}{+}p_{0111} \\
p_{1000}{+}p_{1001}{+}p_{1010}{+}p_{1011} &
p_{1100}{+}p_{1101}{+}p_{1110}{+}p_{1111} 
\end{pmatrix} ,& \\
&
\sigma_{13} \, =  \,
{\rm det} \begin{pmatrix}
p_{0000}{+}p_{0001}{+}p_{0100}{+}p_{0101} &
p_{0010}{+}p_{0011}{+}p_{0110}{+}p_{0111} \\
p_{1000}{+}p_{1001}{+}p_{1100}{+}p_{1101} &
p_{1010}{+}p_{1011}{+}p_{1110}{+}p_{1111} 
\end{pmatrix}, &   
 \ \ {\rm etc}.
\end{eqnarray*}

Zwiernik and Smith \cite{Piotr} gave a semi-algebraic characterization
of the general Markov model on a trivalent phylogenetic tree in terms
of covariances and moments. The statement of their characterization is
somewhat complicated, so we only state a weaker necessary condition
rather than the full characterization.  Specifically, applying
\cite[Theorem 7]{Piotr} to the star tree on $n$ leaves implies the
following result.

\begin{corollary} \label{thm:wishfulthinking}
If a probability distribution $p \in \Delta_{2^n-1}$ lies in the
first RBM model $M^1_n$ then
all its matrix flattenings 
(as in Theorem \ref{thm:flat}) have rank $\leq 2$ and
$$ \sigma_{ij} \sigma_{ik} \sigma_{jk} \,\, \geq \,\, 0 \qquad
\hbox{for all distinct triples} \,\,\, i,j,k \in \{1,2,\ldots,n\}.$$
\end{corollary}

These inequalities follow easily from the parameterization
(\ref{eq:distribution4}), which yields
$$ \sigma_{ij} \,\, = \,\, \lambda (1-\lambda) (\delta_i - \epsilon_i) (\delta_j - \epsilon_j) \,\frac{\delta_i\delta_j}{\prod_{s=1}^n\delta_s}\,\frac{\epsilon_i\epsilon_j}{\prod_{s=1}^n\epsilon_s}.$$
This factorization also shows that the binomial relations $\sigma_{ij}
\sigma_{kl} = \sigma_{il} \sigma_{jk}$ hold on $M^1_n$. These same
binomial relations are valid for the covariances in factor analysis
\cite[Theorem 16]{FactorAnalysis}, thus further underlining the
analogies between the Gaussian case and the binary case.  Theorem 20
in \cite{Piotr} extends the covariance equations $\sigma_{ij}
\sigma_{kl} = \sigma_{il} \sigma_{jk}$ to a collection of quadratic
binomial equations in all tree-cumulants, which in turn can be
expressed in terms of higher order correlations.  For the star tree,
these equations are equivalent on $\Delta_{2^n-1}$ to the rank $\leq
2$ constraints.  However, for general tree models, the binomial
equations in the tree-cumulants are necessary conditions for
distributions to lie in these models.

We now turn to the tropical versions of the RBM model for $k=1$.
The variety $V^1_n$ is cut out 
by the $ 3 {\times} 3$-minors of all flattenings of the
table $\bigl(p(v) \bigr)_{v \in \{0,1\}^n}$. It is known that the
$3 {\times} 3$-minors of {\bf one} fixed two-dimensional matrix form
a tropical basis (cf.~\cite[\S 2]{CTV}).  Indeed, that statement is equivalent to
\cite[Theorem 6.5]{Santos}. It is natural to ask whether the
tropical basis property continues to hold for the set of {\bf all} $3 {\times} 3$-determinants
in Theorem \ref{thm:flat}. Since each  flattening of our table
corresponds to a non-trivial edge split of a tree on $n$ taxa (i.e. a
partition of the set of taxa into two sets each of cardinality $\geq 2$), our question can be reformulated as follows:

\begin{question}\label{conj:intersectionOfEdgeSplits}
Is  the tropical RBM variety $TV^1_n$ equal to the intersection of the tropical 
  rank $2$~varieties
  associated to non-trivial edge splits on a collection of trees on $n$ taxa?
\end{question}

The tropical rank two varieties associated to each of the edge splits
have been studied recently by Markwig and Yu \cite{MY}. They endow
this determinantal variety with a simplicial fan structure that has
the virtue of being shellable. The cones of this simplicial fan
correspond to weighted bicolored trees on $2^{n-1}$ taxa with no
monochromatic cherries. The points in a cone can be viewed as a
matrix encoding the distances between leaves with different
colors in the weighted bicolored tree.

Question~\ref{conj:intersectionOfEdgeSplits} is void for $n \leq 3$,
so the first relevant case concerns $n=4$ taxa.
We were surprised to learn that the answer is negative
already in this case:

\begin{example}
\label{ex:2222}
The prime ideal of the variety $V^1_4$ is generated by 
the $3 \times 3$-minors of the three flattenings of
the $2 {\times} 2 {\times} 2 {\times} 2$-table $p$.
As a statistical model,
each one of the three flattenings corresponds to the graphical model
associated to each one of the quartet trees $(12|34)$, $(13|24)$ and
$(14|23)$, as depicted in Figure~\ref{fig:quartets}.
\begin{figure}[ht]
  \centering
  \subfloat[$(12|34)$]{\includegraphics[scale=0.17]{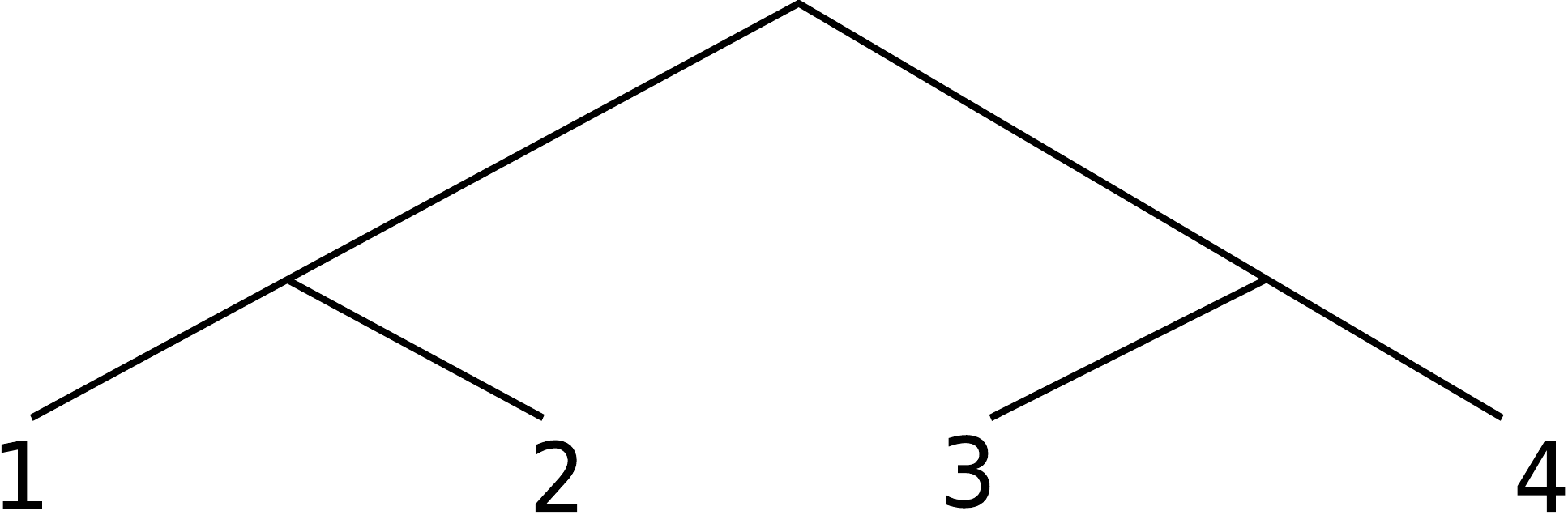}\label{subfig:12|34}}\qquad
\subfloat[$(13|24)$]{\includegraphics[scale=0.17]{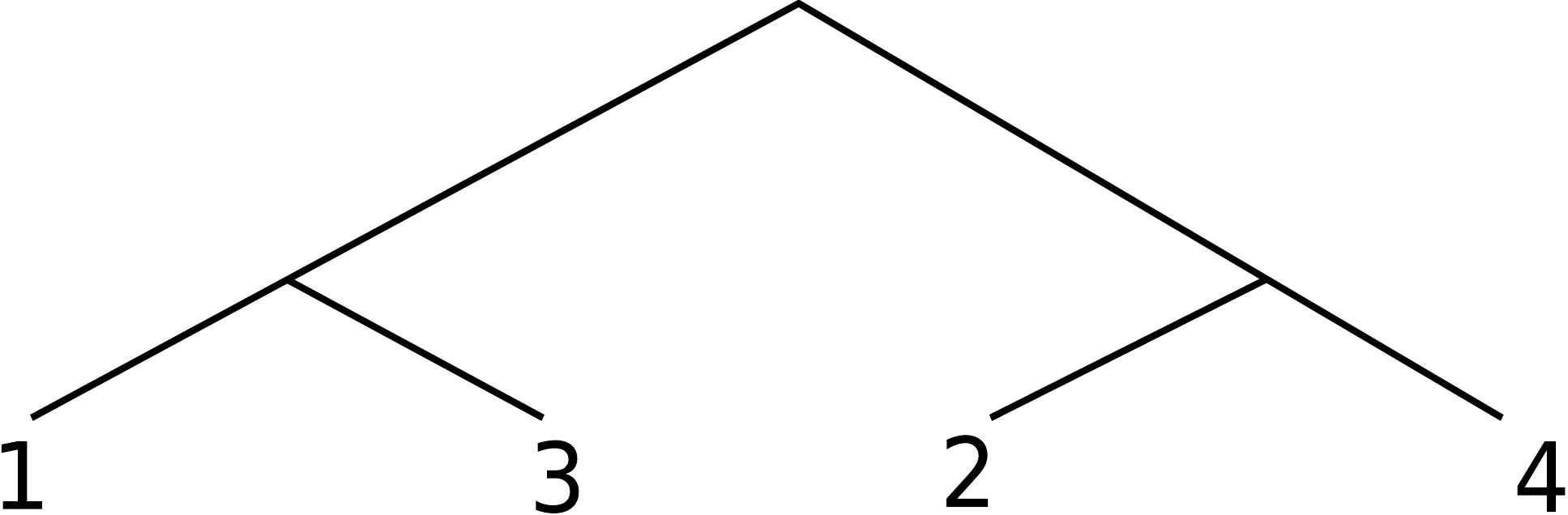}\label{subfig:12|34b}}\qquad
\subfloat[$(14|23)$]{\includegraphics[scale=0.17]{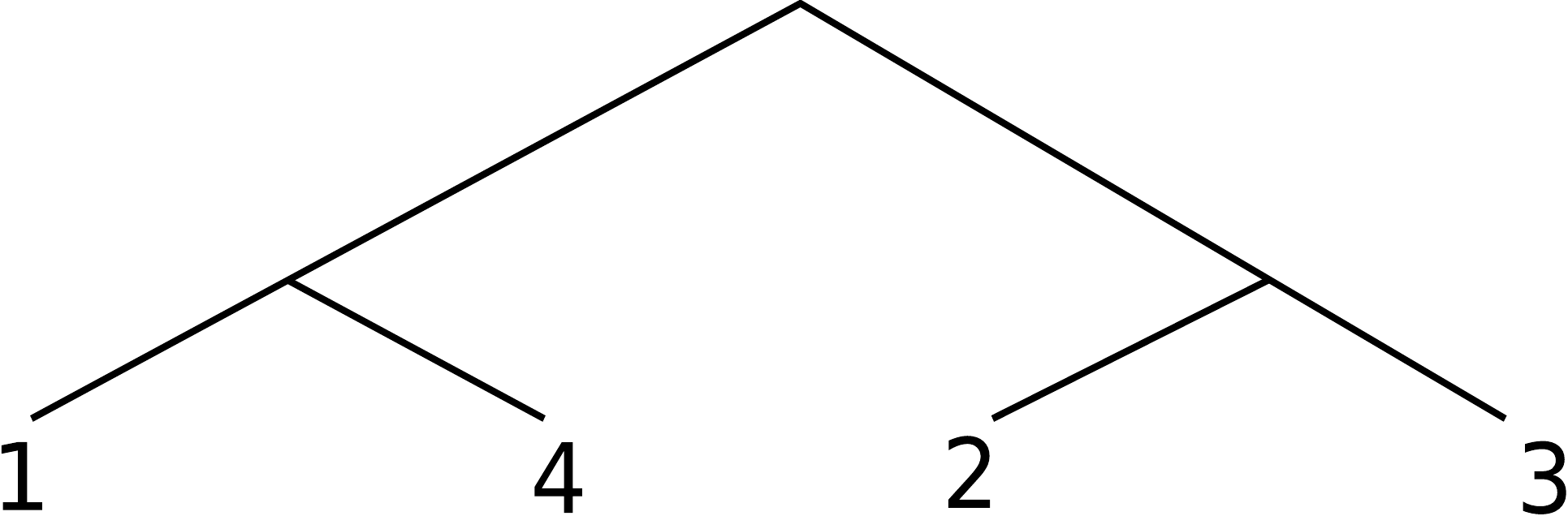}}
  \caption{Quartet trees associated to the flattenings for $n=4$.}
  \label{fig:quartets}
\end{figure}

Algebraically, each flattening corresponds to the variety cut out by the $3 \times 3$-minors
of a $4 \times 4$-matrix of unknowns. These minors form a tropical basis.
The tropical variety they define is a pure fan of dimension $11$ in
$\TP^{15}$ with a 6-dimensional lineality space. The
simplicial fan structure on this variety given by \cite{MY} has
the $f$-vector $\bigl(98, 1152, 4248, 6072, 2952\bigr)$. 
Combinatorially, this object is a shellable
$4$-dimensional simplicial complex which is the bouquet of
$73$ spheres. However,
this determinantal variety admits a different fan structure, induced from
the Gr\"obner fan as in \cite{CTV}, or from the 
fact that the sixteen $3 \times 3$-minors form a tropical basis.
 Its $f$-vector is $\bigl(50, 360, 1128, 1680, 936 \bigr)$.

The tropical variety $TV^1_4$ is a pure fan of dimension
$9$ in $\TP^{15}$. Its lineality space has dimension $4$, and the
cones of various dimensions are tallied in its $f$-vector
$$f(TV^1_4) \quad = \quad \bigl(382, 3436, 11236, 15640, 7680\bigr). $$
Question~\ref{conj:intersectionOfEdgeSplits} asks whether the
$9$-dimensional tropical variety $TV^1_4$ is the intersection of
the three $11$-dimensional tropical determinantal varieties
associated with the three trees in Figure \ref{fig:quartets}.
The answer is ``no''. Using the software {\tt Gfan} \cite{gfan},
we computed the tropical prevariety cut out by the
union of all forty-eight $3 {\times} 3$-minors.
The output is a \emph{non-pure} polyhedral fan of dimension $10$ with a
$4$-dimensional lineality space (the same one as of $TV^1_4$), having
$f$-vector $( 298, 2732, 9440, 13992, 7304, 96)$.
The tropical variety $TV^1_4$ is a triangulation of
a proper subfan, and each of the $96$ 10-dimensional maximal cones
lies in the prevariety but not in the variety.
An example of a vector in the relative interior of a maximal cone is
$$ q \,\, = \,\,  (59, 1, 80, 86, 102, 108, 107, 113, 109, 115, 100, 106, 78, 84, 21, 43) .$$
(Here, coordinates are indexed in lexicographic order $p_{0000},
p_{0001}, \ldots, p_{1111}$).  Given the weights $q$, the initial form
of each $3 {\times} 3$-minor of each flattening is a binomial,
however, the initial form of the following polynomial in the ideal of
$V^1_4$ is the underlined monomial:
$$ \begin{matrix}
     \underline{p_{0000} p_{0110} p_{1010} p_{1101}}
-    p_{0010} p_{0100} p_{1000} p_{1111}
+    p_{0010} p_{0100} p_{1001} p_{1110} \\
-    p_{0000} p_{0110} p_{1001} p_{1110} 
-    p_{0001} p_{0110} p_{1010} p_{1100}
+    p_{0000} p_{0010} p_{1100} p_{1111} \\
-    p_{0000} p_{0010} p_{1101} p_{1110}
+    p_{0001} p_{0110} p_{1000} p_{1110}.
\end{matrix}
$$
Anders Jensen performed another computation, using
{\tt Gfan} and {\tt SoPlex}~\cite{SoPlex}, which verified that we
get a tropical basis by augmenting the $3 {\times} 3$-minors
with the above quartic and its images under the symmetry group of
the $4$-cube. This is a non-trivial
computation because the corresponding fan structure on 
$TV^1_4$ has the $f$-vector
$$  (37442 ,321596 ,843312 ,880488 ,321552). $$
Using the language of \cite{Santos}, we may conclude from our
computational results
that the notions of tropical rank and Kapranov rank disagree
for $2 {\times} 2 {\times} 2 {\times} 2$-tensors.
\qed
\end{example}

\smallskip

Last but not least, we examine the tropical model $TM^1_n $.
This is a proper subfan of  the tropical variety $TV^1_n$,
 namely, $TM^1_n$ is the image of the tropical morphism  $\Phi :
\R^{2n+1} \rightarrow \TP^{2^n-1}$ which is the specialization 
of (\ref{eq:Phi}) for $k=1$. Equivalently, $\Phi$ is
the tropicalization of the map (\ref{eq:distribution4}), and its
coordinates are written explicitly as 
\begin{equation}
\label{qformula2}
 q(v) \,\,\, = \,\,\, b^{\top} v \,\,+\,\,
 {\rm max} \bigl\{\,
0  \,,\,   \omega v + c \bigr\}. 
 \end{equation}
 This concave function is the maximum of two linear functions.
The $2n+1$ parameters are given by a column vector $b \in \R^n$, a row vector $\omega \in \R^n$,
and a scalar $c \in \R$. A different -- but entirely equivalent -- tropicalization can be
derived from (\ref{eq:distribution3}).
As $v$ ranges over $\{0,1\}^n$, there are $2^n$ such concave functions,
and these form the coordinates of  the tropical morphism $\Phi$.
 We note that $\Phi$ made its first explicit appearance in 
   \cite[Equation (10)]{TGSM}, where it was
discussed in the context of {\em ancestral reconstruction}
in statistical phylogenetics. Subsequently, Develin \cite{Develin}
and Draisma \cite[\S 7.2]{tropicalDraisma} introduced
a tropical approach to secant varieties of toric varieties, and our model
fits well into the context developed by these two authors.

\begin{remark}
  The first tropical RBM model $TM^1_n$ is the image of the tropical
  secant map for the Segre variety $(\PP^1)^n$ in the sense of Develin
  \cite{Develin} and Draisma \cite{tropicalDraisma}. The linear space
  for their constructions has basis $\{ \sum_{\alpha \in
    \{0,1\}^n, \alpha_i=1} e_{\alpha} : i =1, \ldots, n\}$, and the
 underlying point configuration consists of the vertices of the $n$-cube.
\end{remark}

In light of Example \ref{ex:2222}, it makes
sense to say that the $2 {\times} \cdots {\times} 2$-tensors in
the tropical variety $TV^1_n$ are precisely those
that have {\em Kapranov (tensor) rank} $\leq 2$. This would be consistent with the
results and nomenclature in \cite{Develin, Santos}. A proper subset of the tensors
of Kapranov rank $\leq 2$ are those that have {\em Barvinok (tensor) rank} $\leq 2$. These
are precisely the points in the first tropical
RBM model $TM^1_n$. 

We close this section by showing that $TM^1_n$ 
has the expected dimension:

\begin{proposition} \label{prop:RBM1}
The dimension of the tropical RBM model $TM^1_n$ is $2n+1$.
\end{proposition}

\begin{proof}
Each region of linearity of the map $\Phi$
is defined by a partition $C$ of $\{0,1\}^n$ into two
disjoint subsets $C^-$ and $C^+$, according to the condition $\omega v+c<0$
or $\omega v+c>0$. Thus, the corresponding region 
is an open convex polyhedral cone, possibly empty,
in the parameter space $\R^{2n+1}$.  It consists of all triples $(b,\omega,c)$ such that
$\,\omega v + c <  0$ for $v \in C^-$ and $\,\omega v + c > 0$ for $v \in C^+$. 
 Assuming $n \geq 3$, we can choose a partition $C$ 
 of $\{0,1\}^n$ such that this  cone  is non-empty and
both $C^-$ and $C^+$ affinely span $\R^n$. The image of the cone
 under the map $\Phi$ spans a space isomorphic
to the direct sum of the images of $\,b \mapsto ( b^{\top}  v : v \in C ) \,$ and
$\,(\omega,c) \,\mapsto \,( \omega v + c :  v \in C^+ )$. Hence this image
has dimension $2n+1$, as expected.
\end{proof}

An illustration of the proof of Proposition \ref{prop:RBM1}
is given in Figure \ref{fig:slicingN=3}. The technique of
partitioning the vertices of the cube will be essential
in our dimension computations for general $k$  in
the next section. In Section~5 we return to
the small models $TM^1_n$ and take a closer look at their
geometric and statistical properties.

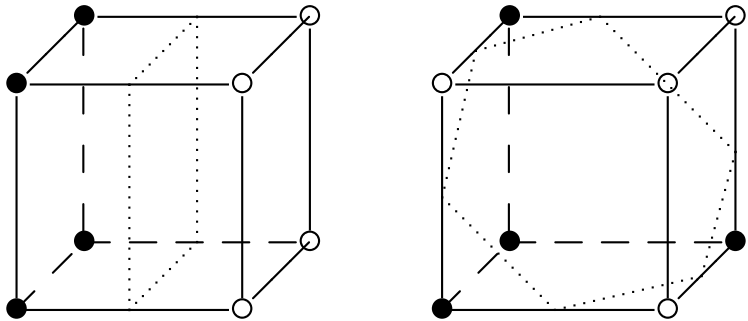
\begin{figure}[ht]
\begin{center}
\[
\scalebox{2}{\begin{xy}<15mm,0mm>:
(0,0) *{\bullet};
p+(0,1) *{\bullet} **@{-};
  p+(.3,.3) **@{-},
p+(1,0) *{\circ} **@{-};
  p+(.3,.3) **@{-},
p+(0,-1) *{\circ} **@{-};
  p+(.3,.3) **@{-},
p+(-1,0) **@{-};
(.3,.3) *{\bullet} **@{--};
p+(0,1) *{\bullet} **@{--};
p+(1,0) *{\circ} **@{-};
p+(0,-1) *{\circ} **@{-};
p+(-1,0) **@{--};
(.5,0);
p+(0,1) **@{.};
p+(.3,.3) **@{.};
p+(0,-1) **@{.};
p+(-.3,-.3) **@{.};
\end{xy}}
\qquad \qquad
\scalebox{2}{\begin{xy}<15mm,0mm>:
(0,0) *{\bullet};
p+(0,1) *{\circ} **@{-};
  p+(.3,.3) **@{-},
p+(1,0) *{\circ} **@{-};
  p+(.3,.3) **@{-},
p+(0,-1) *{\circ} **@{-};
  p+(.3,.3) **@{-},
p+(-1,0) **@{-};
(.3,.3) *{\bullet} **@{--};
p+(0,1) *{\bullet} **@{--};
p+(1,0) *{\circ} **@{-};
p+(0,-1) *{\bullet} **@{-};
p+(-1,0) **@{--};
(.15,1.15);
(.7,1.3)  **@{.};
(1.3,.7)  **@{.};
(1.15,.15)  **@{.};
(.5,0)  **@{.};
(0,.5)  **@{.};
(.15,1.15) **@{.};
\end{xy}}
\]
\end{center}
\caption{Partitions of $\{0,1\}^3$ that define non-empty cones on which $\Phi$ is linear.
Here $C^+$ and $C^-$ are indicated by black ($\bullet$) and white ($\circ$) vertices of the $3$-cube.
The slicing on the right represents a cone
in the parameter space whose image under $\Phi$ is full-dimensional, while the one  on the left
does not.} \label{fig:slicingN=3}
\end{figure}

\section{The Tropical Model and its Dimension}
\label{sec:tropical-model-its}

This section is concerned with Conjecture~\ref{conj:main2} which
states that the tropical RBM model has the expected dimension.
Namely, our aim is to show that
$$\,  {\rm dim}(TM^k_n) \,= \, kn+k+n\, \quad \hbox{ for } \quad k \,\leq\, \frac{2^n-1-n}{n+1}. $$
For $k=1$ this is Proposition \ref{prop:RBM1}, and we now consider the
general case $k \geq 2$.  Our main tool towards this goal is the 
dimension formula in Theorem \ref{thm:cornercuts} below. 
As in the previous section, we study the regions of linearity of the
tropical morphism $\Phi$. 

Let $A$ denote the matrix of format $2^n\times n$ whose rows
are the vectors in $\{0,1\}^n$.
A subset $C$ of the vertices of the $n$-cube is a {\em slicing} if
there exists a hyperplane that has the vertices in $C$ on the positive side
and the remaining vertices of the $n$-cube on the other side. 
In the notation in the proof of Proposition~\ref{prop:RBM1}, the subset
$C$ was denoted by $C^+$. Two examples of slicings
for $n=3$ are shown  in Figure~\ref{fig:slicingN=3}.

For any slicing $C$ of the $n$-cube, let $A_C$  be the $2^n \times (n{+}1)$-matrix
whose rows $v$ indexed by the vertices in $C$ are $(1,v)\in \{0,1\}^{n+1}$
and whose other rows are all identically zero.  The following result
extends the argument used for Proposition~\ref{prop:RBM1}.

\begin{lemma}\label{lem:matrixForTropicalMorphism}
On each region of linearity, the tropical morphism $\Phi$ in (\ref{eq:Phi}) coincides
with the linear map represented by a $2^n\times (nk+n+k)$-matrix of the
form
 $$ \mcA \,\,\, = \,\,\, \bigl(
 \, A \,\,|\,\, A_{C_1} \,\,|\,\, A_{C_2} \,\,|\,\, \cdots \,\,|\,\,
 A_{C_k} \,\bigr) , $$ 
 for some slicings $C_1, C_2, \ldots, C_k$ of
 the $n$-cube.
\end{lemma}

\begin{proof}
The tropical map $\Phi : \R^{nk+n+k} \rightarrow \TP^{2^n-1}$ can be 
written as follows:
$$
    \Phi(W,b,c) \quad = \quad
    \bigl(
    \max_{h\in \{0,1\}^k}\{h^{\top}(Wv+c), 0\} \,+\, b^{\top}v
    \,\,\,
        \bigr)_{v\in
      \{0,1\}^n} .
      $$
Consider a parameter vector $\theta$ with coordinates
$$\theta :=(b_1, b_2, \ldots,
b_n, c_1, \omega_{11}, \ldots, \omega_{1n}, c_2, \omega_{21}, \dots,
\omega_{2n}, \ldots, c_k, \omega_{k1}, \ldots, \omega_{kn}).$$ We
associate to this vector the $k$ hyperplanes $H_i(\theta) = \{v \in
\R^n \,:\, \omega_{i1}v_1 + \ldots +\omega_{in}v_n+c_i=0 \}$ for $i =
1,2,\ldots,k$.  Let us assume that $\theta$ is chosen generically.
Then, for each index $i$, we have $\{0,1\}^n\cap
H_i(\theta)=\emptyset$, and we obtain a slicing of the $n$-cube with
$\,C_i(\theta) := \bigl\{ \,v \in \{0,1\}^n \,:
\,\sum_{j=1}^n\omega_{ij} v_j + c_i >0 \,\bigr\}$.  The generic
parameter vector $\theta$ lies in a unique open region of linearity of
the tropical morphism $\Phi$. More precisely, this region 
corresponds to the cone of all $\theta'$ in $\R^{nk+n+k}$ such that
$C_i(\theta) = C_i(\theta')$ for $i=1,2,\ldots,k$.  By construction,
the map $\Phi: \R^{nk+n+k} \rightarrow \R^{2^n}$ is linear on this
cone. Following the definition of $\Phi$ we see that this linear map
is just left multiplication of the vector $\theta$ by a matrix whose
rows are indexed by the observed states $v$ and columns indexed by the
coordinates of $\theta$.  This matrix is precisely the matrix $\mcA$
above, where $C_i = C_i(\theta)$ for $i=1,2,\ldots,k$. The result
follows by continuity of the map~$\Phi$.
\end{proof}

As an immediate consequence of Lemma \ref{lem:matrixForTropicalMorphism}
we obtain the following result:

\begin{theorem} \label{thm:cornercuts} The dimension of the tropical
  RBM model $\,TM^k_n\,$ 
  equals the maximum rank of any matrix of size $2^n \times
  \bigl( nk+n+k \bigr)$ of the form
$$ \mcA \,\, = \,\, \bigl(
\, A \,\,|\,\, A_{C_1} \,\,|\,\, A_{C_2} \,\,|\,\, \cdots \,\,|\,\, A_{C_k} \,\bigr) , $$
where $\{C_1, C_2, \ldots, C_k\}$ is any set of $k$ slicings of the
$n$-cube.
\end{theorem}

Theorem \ref{thm:cornercuts} furnishes a tool to attack
Conjecture~\ref{conj:main2}. What remains is the
combinatorial problem of finding a suitable collection 
of slicings of the $n$-cube. In what follows we shall apply existing 
results from coding theory to this problem.

There are two quantities from the coding theory literature
\cite{Best77, CoveringCodes, CoverThomas, HuffmannPlessFECC} that are
of interest to us. The first one is $A_2(n,3)$, the size (number of
codewords) of the largest binary code on $n$ bits with each pair of
codewords at least Hamming distance (number of bit flips) $3$ apart.
The second one is $K_2(n,1)$, the size of the smallest {\it covering
  code} on $n$ bits.  In other words, $K_2(n,1)$ is the least number
of codewords such that every string of $n$ bits lies within Hamming
distance one of some codeword. We obtain:

\begin{corollary} \label{cor:codingconn}
The dimension of the tropical RBM model satisfies
\begin{itemize}
\item $\dim TM^k_n \,=\, nk+n+k$ \ for $k < A_2(n,3)$,
\item $\dim TM^k_n \,=\, \min\{nk+n+k,2^n-1\}$ \ for $k = A_2(n,3)$,
\item  $\dim TM^k_n \,\, = \,\,2^n-1$ \ \ \ \ \ for $\,k \geq K_2(n,1)$.
\end{itemize}
\end{corollary}

\begin{proof}
  For the first statement, let $k \leq A_2(n,3)-1$ and fix a code with
  minimum distance $\geq 3$. For each codeword let $C_j$ denote its
  Hamming neighborhood, that is, the codeword together with all strings that
  are at Hamming distance $1$. 
 These $k-1$ sets $C_j$, together with
  some Hamming neighborhood in the complement of their union, are
  pairwise disjoint, and each of them corresponds to a a slicing of
  the cube as in Theorem \ref{thm:cornercuts}.  The disjointness of
  the $k$ neighborhoods means that $\,nk+n+k \leq 2^n-1$.  Elementary
  row and column operations can now be used to see that the
  corresponding $2^n \times (nk+n+k)$ matrix $\mcA = (A | A_{C_1} |
  \cdots | A_{C_k})$ has rank $nk+n+k$.  This is because, after such
  operations, $\mcA$ consists of a block of format $n \times n$ and
  $k$ blocks of format $(n+1) \times (n+1)$ along the diagonal. The
  first block has rank $n$ and the remaining $k$ blocks have rank
  $n+1$ each.  The same reasoning is valid for $k = A_2(n,3)$ except
  that it may now happen that $nk+k+n \geq 2^n$. In this case, the 
  $k$ blocks have total rank $k(n+1)$ and together with the first
  $n\times n$ block they give a matrix of maximal rank $\min\{nk+n+k, 2^n-1\}$.

For the third statement, we suppose $C_1, \dots, C_k$ are slicings
with subslicings $C'_i \subseteq C_i$ such that the $C'_i $ are
disjoint and no $n+1$ of the vertices in a given $C_i$ lie in a hyperplane.  Then $\rank(\mcA) \geq n + \sum_{i=1}^k |C'_i|$ by similar arguments.  This is because we may construct the $C_i'$ by pruning neighbors from codewords, and are left with a lower-dimensional Hamming neighborhood which is a slicing.
\end{proof}

The computation of $A_2(n,3)$ and $K_2(n,1)$, both in general and for specific values of $n$, has been an active area of research since the 1950s.  In Table \ref{tab:knownspecialcases} 
we summarize some of the known results for specific values of $n$.  
This table is based on \cite{CoveringCodes, TableofNonlinearBinaryCodes}.
For general values of $n$, the following bounds can be obtained.

\begin{table}
{\small
\begin{tabular}{c|c|c}
$n$ & $k$ $\leq$&$k$ $\geq$\\
\hline
&&\\
5 &$2^2$&$\mathbf{7}$ \\ 
6 &$2^3$ &$\mathbf{12}$\\
7 &$2^4$&$2^4$ \\
8 &$\mathbf{2^2} \cdot \mathbf{5}$&$2^5$ \\
9 &$\mathbf{2^3} \cdot \mathbf{5}$&$\mathbf{62}$\\
10&$\mathbf{2^3}\cdot\mathbf{9}$&$\mathbf{120}$\\
11&$\mathbf{2^4}\cdot\mathbf{9}$&$\mathbf{192}$\\
12&$2^{8}$&$\mathbf{380}$\\
13&$2^{9}$&$\mathbf{736}$\\
14&$2^{10}$&$\mathbf{1408}$\\
15&$2^{11}$&$2^{11}$\\
16&$\mathbf{2^5}\cdot\mathbf{85}$&$2^{12}$\\
17&$\mathbf{2^6}\cdot\mathbf{83}$&$2^{13}$ \\
18&$\mathbf{2^8}\cdot\mathbf{41}$&$2^{14}$\\
19&$\mathbf{2^{12}}\cdot\mathbf{5}$&$\mathbf{31744}$\\
20&$\mathbf{2^{12}}\cdot\mathbf{9}$&$\mathbf{63488}$\\
21&$\mathbf{2^{13}}\cdot\mathbf{9}$&$\mathbf{122880}$\\
22&$\mathbf{2^{14}}\cdot\mathbf{9}$&$\mathbf{245760}$\\
23&$\mathbf{2^{15}}\cdot\mathbf{9}$&$\mathbf{393216}$\\
24&$2^{19}$ &$\mathbf{786432}$\\
25&$2^{20}$ &$\mathbf{1556480}$\\
26&$2^{21}$ &$\mathbf{3112960}$\\
27&$2^{22}$ &$\mathbf{6029312}$\\
28&$2^{23}$ &$\mathbf{12058624}$\\
29&$2^{24}$ &$\mathbf{23068672}$\\
30&$2^{25}$ &$\mathbf{46137344}$\\
31&$2^{26}$&$2^{26}$ \\
32&$\mathbf{2^{20}}\cdot\mathbf{85}$ &$2^{27}$ \\
33&$\mathbf{2^{21}}\cdot\mathbf{85}$&$2^{28}$\\
\end{tabular}
\qquad
\begin{tabular}{c|c}
$n$ & $k$ $\leq$\\
\hline
&\\
35&$\mathbf{2^{23}}\cdot\mathbf{83}$\\
37&$\mathbf{2^{26}}\cdot\mathbf{41}$\\
39&$\mathbf{2^{31}}\cdot\mathbf{5}$\\
47&$\mathbf{2^{38}}\cdot\mathbf{9}$\\
63&$2^{57}$\\
70&$\mathbf{2^{43}}\cdot \mathbf{1657009}$\\
71&$\mathbf{2^{63}}\cdot\mathbf{3}$\\
75&$\mathbf{2^{63}}\cdot\mathbf{41}$\\
79&$\mathbf{2^{70}}\cdot\mathbf{5}$\\
95&$\mathbf{2^{85}}\cdot\mathbf{9}$\\
127&$2^{120}$\\
141&$\mathbf{2^{113}}\cdot\mathbf{1657009}$\\
143&$\mathbf{2^{134}}\cdot\mathbf{3}$\\
151&$\mathbf{2^{138}}\cdot\mathbf{41}$\\
159&$\mathbf{2^{149}}\cdot\mathbf{5}$\\
163&$\mathbf{2^{151}}\cdot\mathbf{19}$\\
191&$\mathbf{2^{180}}\cdot\mathbf{9}$\\
255&$2^{247}$\\
270&$\mathbf{2^{202}}\cdot \mathbf{1021273028302258913}$\\
283&$\mathbf{2^{254}}\cdot \mathbf{1657009}$\\
287&$\mathbf{2^{277}}\cdot \mathbf{3}$\\
300&$\mathbf{2^{220}}\cdot\mathbf{3348824985082075276195}$\\
303&$\mathbf{2^{289}}\cdot\mathbf{41}$\\
319&$\mathbf{2^{308}}\cdot\mathbf{5}$\\
327&$\mathbf{2^{314}}\cdot\mathbf{19}$\\
383&$\mathbf{2^{371}}\cdot\mathbf{9}$\\
511&$2^{502}$\\
512&$\mathbf{2^{443}}\cdot\mathbf{1021273028302258913}$\\
&\\
\end{tabular}
}
\caption{Special cases where Conjecture 
\ref{conj:main2} holds, based on \cite{CoveringCodes, TableofNonlinearBinaryCodes} and Corollary \ref{cor:codingconn}.  Bold entries show improvements made by various researchers on the bounds provided by Corollary \ref{cor:codingbounds}.   For example, for $n=19$, $TM^k_n$ has the expected dimension if $k \leq 2^{12}\cdot 5=20480 $ and dimension $2^n-1$ if $k \geq 31744$, while the Corollary \ref{cor:codingbounds} bounds are $2^{14}=16384$ and $2^{15}=32768$, respectively.  The $k\leq$ column lists lower bounds on $A_2(n,3)$ while the $k \geq$ column lists upper bounds on $K_2(n,1)$. 
} \label{tab:knownspecialcases}
\end{table}

\begin{proposition} \label{prop:codebounds}
For binary codes with $n \geq 3$, the Varshamov bound holds:
\[
A_2(n,3) \,\,\geq \,\, 2^{n - \ceil{\log_2(n+1)}}.
\]
For covering codes,  the following inequality holds:
\[
K_2(n,1) \,\,\leq \,\, 2^{n-\floor{\log_2(n+1)}}.
\]
For $n =2^{\ell}-1$ with $ \ell \geq 3 $,
we have the equality $\,A_2(n,3) = K_2(n,1)=2^{2^{\ell}-\ell-1}$.
\end{proposition}

\begin{proof}
  A proof of the Varshamov bound on $A_2(n,3)$ may be found in
  \cite{HuffmannPlessFECC}.  The last statement holds because
  $A_2(n,3)=K_2(n,1)$ for perfect Hamming codes: for every $\ell \geq
  3$ there is a perfect $(2^{\ell}-1,2^{\ell}-\ell -1,3)$ Hamming code
  (i.e. a perfect Hamming code on $2^{\ell}-1$ bit, of size
  $2^{\ell}-\ell -1$, and with Hamming distance $3$). For a proof of
  this result, see \cite{CoverThomas}. Additionally, we have
  $K_2(2^m-1,1)=2^{2^m-m-1}$ for $m \geq 3$; see \cite{CoveringCodes}.

  The simple upper bound on $K_2(n,1)$ can be obtained by using
  overlapping copies of the next smallest Hamming code.  Suppose $n
  \neq 2^{\ell'}-1$ for any $\ell'$, i.e. $n$ is between \emph{Hamming
    integer numbers} (i.e.~integers of the form $2^{\ell}-1$).  Let $\un$ be
  the next smallest Hamming integer $n$, with
  $\ell=\floor{\log_2(n+1)}$, so $\un = 2^{\ell}-1$.  The number of
  hidden nodes needed to cover the $\un$-cube is exactly $K(\un,1) =
  2^{{2^{\ell}} -\ell - 1}$.  We may use the $\un$ codes to cover each
  of the $2^{n-\underline{n}}$ faces of the $n$-cube with
  $2^{\underline{n}}$ vertices, although we will have overlaps.  That
  is,
\begin{equation} \label{eq:inbetweenk1}
K_2(n,1) \, \leq \, K_2(\un,1) \cdot 2^{n-\un}.
\end{equation}
Taking $\log_2$ in the inequality \eqref{eq:inbetweenk1}, we obtain 
$$ \log_2  K_2(n,1) \,\,\leq\,\,  \log_2 ( K_2(\un,1)2^{n-\un}) \,\,= \,\,
                             n-\floor{\log_2(n+1)}.
$$
This implies $\,
K_2(n,1)\,\, \leq  \,\, 2^{n- \floor{\log_2(n+1)}}$.
\end{proof}

Our method results in the following upper and lower bounds for arbitrary values of $n$.
Note that the bound is tight if $n+1$ is a power of $2$. Otherwise
there might be a multiplicative gap of up to $2$ between the lower and upper bound.
In addition to these
general bounds, we have the specific results recorded in Table \ref{tab:knownspecialcases}.
 
\begin{corollary} \label{cor:codingbounds}
The coding theory argument leads to the following bounds:
\begin{itemize}
\item If $\,k \,<\, 2^{n - \ceil{\log_2(n+1)}}$, then $\,\dim TM^k_n = nk+n+k$.
\item If $\,k \,=\, 2^{n - \ceil{\log_2(n+1)}}$, then $\,\dim TM^k_n = \min\{nk+n+k, 2^n-1\}$.
\item  If $\,k \,\geq \, 2^{n-\floor{\log_2(n+1)}}$, then $\,\dim TM^k_n = 2^n-1$.
\end{itemize}
\end{corollary}

\begin{proof}[Proof of Theorem \ref{thm:looksreasonable}]
This is now easily completed
by combining Corollary \ref{cor:codingbounds}
with the inequalities in (\ref{InEqChain}).
\end{proof}

We close this section with the remark that the use of Hamming codes
is a standard tool in the study of dimensions of secant varieties.
We learned this technique from Tony Geramita. For a review of the
relevant literature see \cite{tropicalDraisma}. It is important to note that,
in spite of the combinatorial similarities, the varieties we study here are different
and more complicated than higher secant varieties of Segre varieties.

\section{Polyhedral Geometry of Parametric Inference}
\label{sec:polyh-geom-param}
The tropical model $TM^k_n$ is not just a
convenient tool for estimating the dimension of the statistical model $M^k_n$.  It is also 
of interest as the geometric object that organizes the space of inference functions which the model can compute.  This statistical interpretation of tropical spaces was introduced in
\cite{TGSM}  and further developed in
\cite{InferenceFunctions, ASCB}. We shall now discuss this perspective 
for the RBM model.

Given an RBM model with fixed parameters learned by some estimation procedure and an
observed state $v$, we want to infer which value $\hat{h}$ of the hidden data  maximizes $\text{Prob}(h\mid v)$.  The inferred string $\hat{h}$  might be used in classification or as the input data for another RBM in a deep architecture.  Such a vector of hidden states
is called an \emph{explanation} of the observation $v$. Each choice of
parameters $\theta=(b,W,c)$ defines an \emph{inference function}
$I_\theta $ sending $v \mapsto \hat{h}$. 
The value $I_{\theta}(v)$ equals the hidden string $h \in \{0,1\}^k$ 
that attains the maximum in the tropical polynomial
\begin{equation}
 \max_{h\in \{0,1\}^k}\{h^\top W v+c^{\top}h+b^{\top}
v\} \,= \, b^{\top} v \,\, + \! \max_{h\in \{0,1\}^k}\{h^\top W v +c^\top h\}.\label{eq:tropPsi}
\end{equation}

In order for the inference function $I_\theta$ to be well-defined,
it is necessary (and sufficient) that
$\theta = (b,W,c)$ lies in an open cone of linearity of the
tropical morphism $\Phi$. In that case, the maximum in
equation~\eqref{eq:tropPsi} is attained for a unique value of $h$. That $h$
can be recovered from the expression of $\Phi$ as we vary the
parameters in the fixed cone of linearity. Thus, the inference
functions are in one-to-one correspondence with the regions of
linearity of the tropical morphism $\Phi$.

The RBM model grew out of work on artificial neurons modeled as linear threshold functions \cite{MinksyPapert, Rosenblatt}, and we pause our geometric discussion to make a few remarks about these functions and the types of inference functions that our model can represent.  
A \emph{linear threshold function} is a function
$\{0,1\}^n\to \{0,1\}$ defined by choosing a weight vector
$\omega$ and a target weight $\pi$. For any point $v\in \{0,1\}^n$ we
compute the value $\omega  v$, we test if this quantity is at
most $\pi$ or no, and we assign value $0$ or $1$ to $v$ depending on $\pi \geq
\omega  v$ or $\pi < \omega v$. The weights $\omega, \pi $ define a 
hyperplane in $\R^n$ such that the vertices of the $n$-cube lie on the ``true'' or
``false'' side of the hyperplane. Using the linear threshold functions, we construct
a $k$-valued function $\{0,1\}^n\to \{0,1\}^k$ where we replace the
weight vector $\omega$ by a $k\times n$ matrix $W$ and the target
weight $\pi $ by a vector $\pi \in \R^k$. More precisely, the function
assigns a vertex of the $k$-cube where the $i$-th coordinate equals 0
if $(Wv)_i\geq \pi_i$ and $1$ if not.  
Our discussion of slicings of the $n$-cube in Section~4 implies
the following observation:

\begin{proposition}
The inference functions for the restricted Boltzmann machine model
  $M^k_n$ are precisely those Boolean functions $\{0,1\}^n
  \rightarrow \{0,1\}^k$ for which each of the $k$ coordinate
  functions  $\{0,1\}^n \rightarrow \{0,1\}$  is a linear threshold function.
\end{proposition}

Most Boolean functions are not linear threshold functions,
that is, are not inference functions for the model $M^1_n$. For example, the 
parity function cannot be so represented.  To be precise, while the number
of all Boolean functions is $2^{2^n}$, 
it is known \cite{Ojha} that for $n \geq 8$
the number $\lambda(n)$ of
linear threshold functions satisfies 
$$2^{{n \choose 2} +16}< \lambda(n)  \leq 2^{n^2}.$$
The exact number $\lambda(n)$ of linear threshold functions has been computed 
for up to $n=8$. 
The {\em On-Line Encyclopedia of Integer Sequences}~\cite[A000609]{OEIS}
reveals
 \begin{equation}
\label{LTF}
\lambda(1 \ldots 8)\, = \,  4, 14, 104, 1882, 94572, 15028134, 8378070864, 17561539552946. \!\!
 \end{equation}
Combining $k$ such functions for $k \geq 2$ yields $\lambda(n)^k=2^{\Theta (kn^2)}$ 
possible inference functions for the RBM model $M^k_n$. This number
grows exponentially in the number of model parameters. This is consistent
with the result of Elizalde and Woods in \cite{InferenceFunctions}
which states that the number of inference functions of a graphical model
grows polynomially in the size of the graph when the number of
parameters is fixed.

In typical implementations of RBMs using  IEEE 754 doubles, the size in bits of the representation is $64(nk+n+k)$.  Thus the number $2^{\Theta (kn^2)}$ of inference functions representable by a theoretical RBM $M^k_n$ will eventually outstrip the number $2^{64(nk+n+k)}$ representable in a fixed-precision implementation; for example with $k=100$ hidden nodes, this happens at $n\geq 132$.  As a result, the size of the regions of linearity will shrink to single points in floating point representation.  This is one possible contributor to the difficulties that have been encountered in scaling RBMs.

The tropical point of view allows us to organize the geometric
information of the space of inference functions into the tropical
model $TM^k_n$, which can then be analyzed with the tools of tropical
and polyhedral geometry.  We now describe this geometry in the case
$k=1$.  Geometrically, we can think of the linear threshold functions
as corresponding to the vertices of the $(n+1)$-dimensional zonotope
corresponding to the $n$-cube.  This zonotope is the Minkowski sum in
$\R^{n+1}$ of the $2^n$ line segments $[(1,{\bf 0}),(1,v)]$ where $v$ ranges
over the set $\{0,1\}^n$.

The quantity $\lambda(n)$ is the number of vertices
of these zonotopes, and their facet numbers were  computed by
Aichholzer and Aurenhammer \cite[Table
2]{HyperplanesInHypercubes}. They are
\begin{equation}
\label{auren}
  4, 12, 40, 280, 6508, 504868, 142686416, 172493511216
  , \ldots \quad \qquad \quad  \end{equation}
  For example, the second entry in  (\ref{LTF}) and (\ref{auren})
   refers to a $3$-dimensional
  zonotope known as the {\em rhombic dodecahedron},
  which has $12$ facets and $\lambda(2) = 14$ vertices.
 Likewise, the third entry in   (\ref{LTF}) and (\ref{auren})
 refers to a $4$-dimensional zonotope
  with $40$ facets and $\lambda(3) = 104$ vertices.
The normal fan of that zonotope is an arrangement of eight hyperplanes,
indexed by $\{0,1\}^3$,
 which partitions $\R^4$ into $104$ open convex
 polyhedral cones.
 That partition lifts to a partition of the parameter space $\R^7$
 for $M^1_3$ whose cones are precisely the regions
 on which the tropical morphism $\Phi$  is linear.
 The image of that morphism is the 
 first non-trivial tropical RBM model $TM^1_3$. 
 This model has the expected dimension $7$ and it happens to be a pure fan.
 
\begin{figure}[ht]
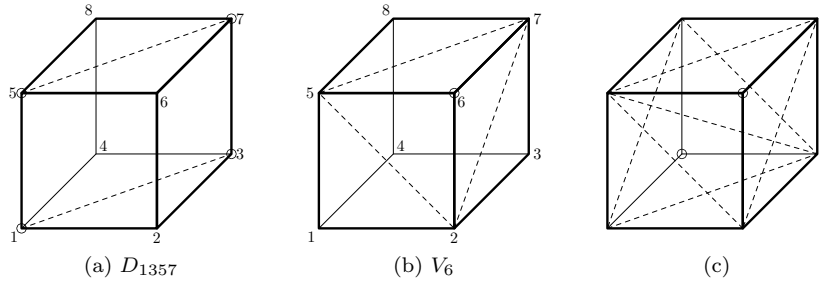

\subfloat[$D_{1357}$]
{\includegraphics[scale=0.6]{triangulation3Cube1.mps}
    \label{subfig:1}}\qquad
\subfloat[$V_6$]
{\includegraphics[scale=0.6]{triangulation3Cube0.mps}\label{subfig:0}}\qquad
\subfloat[]
{\includegraphics[scale=0.6]{triangulation3Cube2.mps}
\label{subfig:2}
}\caption{Subdivisions of the 3-cube that represent vertices
and facets of $TM^1_3$}\label{fig:subdivisions3Cube}
\end{figure}

\begin{figure}[ht]
  \centering
  \includegraphics{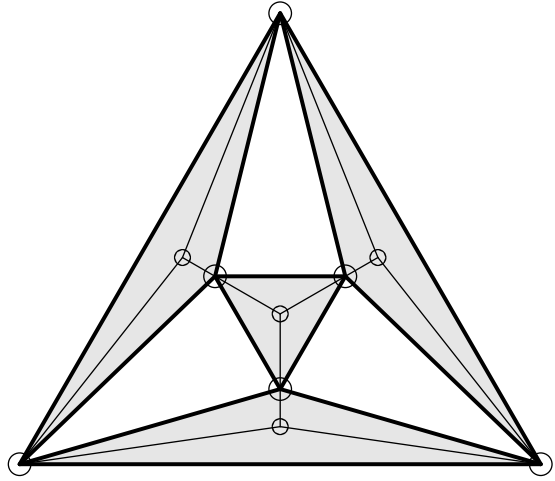}
  \caption{The tropical model $TM^1_3$ is glued
  from four triangulated bipyramids.
In this octahedron graph, each
  of the bipyramids is represented by
  a shaded triangle.}\label{fig:TM^1_3}.
 \end{figure}

 \begin{example} \label{ex52}
   The tropical RBM model $TM^1_3$ is a $7$-dimensional fan whose
   lineality space is $3$-dimensional. It is a subfan of the secondary
   fan of the $3$-cube \cite[Corollary 2.2]{Develin}.  The secondary
   fan of the $3$-cube can be represented as a $3$-dimensional polyhedral
   sphere with $ f$-vector $(22,100,152,74)$. The $74$ facets of
   that $3$-sphere correspond to
   triangulations of the $3$-cube.  The tropical model $TM^1_3$
   consists of all regular subdivisions of the $3$-cube with two regions
   covering all eight vertices.  It sits inside the polyhedral $3$-sphere
   as a \emph{simplicial} subcomplex with  $f$-vector $(14,40,36,12)$.  
   Its $12$ facets (tetrahedra)
           correspond to a single triangulation type of the
  $3$-cube as depicted in Figure~\ref{subfig:2}.
The $14$ vertices of $TM^1_3$ come in two families: six vertices $D_j$
corresponding to diagonal cuts, as in Figure~\ref{subfig:1}, and eight
vertices $V_i$ representing corner cuts, as in Figure~\ref{subfig:0}. The
edges come in three families: four edges $V_iV_j$ corresponding to
pairs of corner cuts at antipodal vertices of the cube,
twenty-four edges  $V_iD_j$, and twelve edges $D_iD_j$. Finally, 
of the four possible triangles, only two types are present: the ones
with two vertices of different type. Thus, they are $12$ triangles
$V_iV_jD_k$ and $24$ triangles $V_iD_jD_k$. 

Figure~\ref{fig:TM^1_3} depicts the simplicial complex $TM^1_3$
which is pure of dimension $3$. The
six vertices $D_i$ and the twelve edges $D_jD_k$ 
form the edge graph of an octahedron. The four nodes 
interior to the shaded triangles represent pairs of vertices $V_i$
that are joined by an edge. Each of the shaded triangles
represents three tetrahedra that are glued together along a common edge
$V_iV_j$. Thus the twelve tetrahedra in $TM^1_3$ come as  four
triangulated bypiramids. The four bypiramids  are then glued
into four of the triangles in the octrahedron graph.
Our analysis shows that the complex $TM^1_3$
has reduced homology
 concentrated in degree $1$ and it has rank~ $3$.
\qed
 \end{example}

The previous example is based on the fact that
the image of the tropical map $\Phi : \R^{2n+1} \rightarrow \R^{2^n} $ is a subfan
of the secondary fan of the $n$-cube. However, it is important to 
note that $\Phi$ is {\bf not} a morphism of fans with respect to the natural fan
structure on the parameter space $\R^{2n+1}$ given
by the slicings of the $n$-cube.

\begin{example}
Consider the case $n=2$. Here $M^1_{2}$ equals $\R^4$ with
its secondary fan structure coming from the two triangulations of the square.
Modulo lineality, this fan is simply the 
standard fan structure $\{\R_{\leq 0} , \{0\}, \R_{\geq 0}\}$ on the real line.
The fan structure on the parameter space $\R^7$ has $14$ maximal cones.
Modulo lineality, this is the normal fan of the rhombic dodecahedron,
 i.e.~a partition of  $\R^3$ into $14$ open convex cones
 by an arrangement of four  planes through the origin.
       Ten of these $14$ open cones are mapped onto cones,
 namely,  four are mapped onto $\R_{\leq 0}$,
 two are  mapped onto $\{0\}$,  and four onto
 $\R_{\geq 0}$. The remaining four cones are mapped 
 onto $\R^1$, so $\Phi$ does not respect the fan structures
 relative to these four cones.
 
 The situation is analogous for $n=3$ but more complicated.  The
 tropical map $\Phi$ is injective on precisely eight of the $104$
 maximal cones in the parameter space. These eight cones are the
 slicings shown on Figure~\ref{subfig:1}.  The map $\Phi$ is injective
 on such a cone, but the cone is divided into three subcones by the
 secondary fan structure on $M^1_3$.  The resulting $24 = 3 \cdot 8$
 maximal cells in the parameter space are mapped in a 2-to-1 fashion
 onto the $12$ tetrahedra in Figure~\ref{fig:TM^1_3}.  It would be
 worthwhile to study the combinatorics of the graph of $\,\Phi\,$ for
 $n \geq 3$.  \qed
 \end{example}

\section*{Acknowledgments}
We thank  Jan Draisma, JM Landsberg,  Honglak Lee, Sorgey Norin, Lior Pachter,
Seth Sullivant, Ilya Sutskever, Jenia Tevelev, and Piotr Zwiernik
for helpful discussions. Special thanks go to Anders Jensen
for computations he did for~us.

\bibliographystyle{amsalpha}

\end{document}